\documentclass[5p,times]{elsarticle}

\makeatletter\usepackage{microtype}\g@addto@macro\@verbatim{\microtypesetup{activate=false}}\makeatother

\makeatletter
\def\ps@pprintTitle{%
  \let\@oddhead\@empty
  \let\@evenhead\@empty
  \def\@oddfoot{Author-generated preprint of \textit{Neural Networks} \textbf{173}, no. 106181 (2024). \\
  \url{https://doi.org/10.1016/j.neunet.2024.106181}}
  \let\@evenfoot\@oddfoot
}
\makeatother

\usepackage[british]{babel}
\usepackage{enumerate}
\usepackage{xfrac}	
\usepackage{slashed} 
\usepackage{savesym}
\savesymbol{intertext}
\usepackage{amsmath}
\restoresymbol{TT}{intertext}
\usepackage[bookmarksopen,bookmarksdepth=2]{hyperref}
\usepackage{tikz}
\usepackage{caption,subcaption}

\providecommand{\noopsort}[1]{}

\mathchardef\ordinarycolon\mathcode`\:
\mathcode`\:=\string"8000
\begingroup \catcode`\:=\active
  \gdef:{\mathrel{\mathop\ordinarycolon}}
\endgroup

\newcommand{\N}{\mathbb N}
\newcommand{\Z}{\mathbb Z}

\newcommand{\R}{\mathbb R}

\newcommand{\mR}{\mathcal{R}}
\newcommand{\Cr}{C_\mR}

\newcommand{\fS}{\mathfrak{S}}

\newcommand{\supnorm}[1]{\norm{ #1 }_\infty}

\newcommand{\norm}[1]{\left\| #1 \right\|}

\newcommand{\ran}{\operatorname{ran}}

\newcommand{\spn}{\textnormal{span}}

\renewcommand{\mid}{~\middle|~}

\newcommand{\MLP}[3]
					{\mathcal{N}_\varphi^{#1}(\R^#2\def\mycmd{#3}\if1\mycmd \else ,\R^#3 \fi)}
\newcommand{\supp}{\operatorname{supp}}
\newcommand{\image}[2]{\includegraphics[width=#1\textwidth, height=#1\textheight,keepaspectratio]{#2}}

\usepackage{amsthm}
\newtheorem{thm}{Theorem}[section]
\newtheorem{lem}[thm]{Lemma}
\newtheorem{prop}[thm]{Proposition}
\newtheorem{cor}[thm]{Corollary}
\newtheorem{defn}[thm]{Definition}

\newtheorem{remark}[thm]{Remark}

\begin{document}

\title{Noncompact uniform universal approximation}
\author{Teun D. H. van Nuland\\~\\TU Delft, EWI/DIAM, P.O. Box 5031, 2600 GA Delft, The Netherlands\\~\\teunvn$\otimes$gmail$\cdot$com}
\date{\today}

\begin{abstract}
The universal approximation theorem is generalised to uniform convergence on the (noncompact) input space $\R^n$. All continuous functions that vanish at infinity can be uniformly approximated by neural networks with one hidden layer, for all activation functions $\varphi$ that are continuous, nonpolynomial, and asymptotically polynomial at $\pm\infty$. When $\varphi$ is moreover bounded, we exactly determine which functions can be uniformly approximated by neural networks, with the following unexpected results.
Let $\overline{\mathcal{N}_\varphi^l(\R^n)}$ denote the vector space of functions that are uniformly approximable by neural networks with $l$ hidden layers and $n$ inputs. For all $n$ and all $l\geq2$, $\overline{\mathcal{N}_\varphi^l(\R^n)}$ turns out to be an algebra under the pointwise product. 
If the left limit of $\varphi$ differs from its right limit (for instance, when $\varphi$ is sigmoidal) the algebra $\overline{\mathcal{N}_\varphi^l(\R^n)}$ ($l\geq2$) is independent of $\varphi$ and $l$, and equals the closed span of products of sigmoids composed with one-dimensional projections. If the left limit of $\varphi$ equals its right limit, $\overline{\mathcal{N}_\varphi^l(\R^n)}$ ($l\geq1$) equals the (real part of the) commutative resolvent algebra, a C*-algebra which is used in mathematical approaches to quantum theory. 
In the latter case, the algebra is independent of $l\geq1$, whereas in the former case $\overline{\mathcal{N}_\varphi^2(\R^n)}$ is strictly bigger than $\overline{\mathcal{N}_\varphi^1(\R^n)}$.
\end{abstract}
\maketitle

\section{Introduction}

Neural networks can uniformly approximate any continuous function only when the magnitude of the considered input values is bounded by a predetermined constant. Typical universal approximation theorems that use the entire noncompact input space $\R^n$ make use of convergence `uniformly on compacts' \cite{Barron,Cybenko,HKK,Hornik,HSW,LLPS,LWN} or convergence with respect to an integral norm on $\R^n$ \cite{Hornik,KL}. Such theorems do not rule out errors in the approximation growing exponentially (or worse) in the magnitude of the input values.

\textit{Noncompact} and \textit{uniform} approximation -- which uses convergence with respect to the supremum norm over $\R^n$ -- is a much stronger notion. In theory, it allows one to train a network up to a desired precision which is then respected by \textit{all} input values. It also gives a more honest picture of the generalisation capability of neural networks, as we shall see later. 

It is a common misconception that every continuous function on $\R^n$ can be uniformly approximated; in fact many commonplace continuous functions cannot.\footnote{E.g., $\sin(x)$, $e^x$, 
unless the activation function is specially tailored for these.
}
The question remains: precisely \textit{which} functions can be uniformly approximated?

Let the activation function $\varphi:\R\to\R$ be continuous and nonlinear, with asymptotically linear behaviour near $\pm\infty$. 
One-layer neural networks are by definition linear combinations of functions of the form
\begin{align}\label{eq:one-layer ANN}
x\mapsto \varphi(a\cdot x+b)\qquad (a\in\R^n,b\in\R),
\end{align}
where $\cdot$ is the standard inner product on $\R^n$. Such functions are constant in $n-1$ directions. If $n\geq2$, a nonzero one-layer neural network will therefore never be in $C_0(\R^n)$, the space of continuous functions that vanish at infinity,\footnote{If a neural network vanishes (approximately) at infinity, it means that the network responds consistently to large inputs, like outliers.
} no matter the activation function or the amount of nodes.\footnote{See Theorem \ref{thm:0C0}.}
Our first result is that, nonetheless, all functions in $C_0(\R^n)$ are uniformly approximable by one-layer neural networks (and therefore also by arbitrarily deep neural networks).
This generalises the 
universal approximation theorem to a truly noncompact statement.

We also precisely characterise the space of (uniformly) approximable functions in the case that $\varphi$ is moreover bounded.
The above result then implies that the space of approximable functions is some vector space between $C_0(\R^n)$ and the space of bounded continuous functions, $C_\textnormal{b}(\R^n)$.

By giving an explicit characterisation, we shall prove that this vector space is an \textit{algebra} under the usual pointwise operations. Equivalently, products of neural networks are approximable by neural networks. 

This uncovers a novel connection between neural networks and the theory of C*-algebras \cite{Murphy}, as any norm-closed subalgebra of $C_\textnormal{b}(\R^n)$ is a real C*-algebra. 
We do not rely on the theory of C*-algebras in this paper, but one should know that this theory initiated these findings, and might have merit for the machine learning community for reasons discussed in \cite{HWM}.

{\sloppy
Below, we discuss the explicit characterisation of the space of approximable functions, which notably does not depend greatly on the activation function $\varphi$, but only on the question whether $\varphi(-\infty):=\lim_{x\to-\infty}\varphi(x)$ equals $\varphi(\infty):=\lim_{x\to\infty}\varphi(x)$.
}

\subsection[The case phi(-inf)=phi(inf)]{The case $\varphi(-\infty)=\varphi(\infty)$}
We first consider the class of $\varphi$ satisfying $\varphi(-\infty)=\varphi(\infty)$.
We find that, for any amount of hidden layers, the vector space of approximable functions is equal to the real part of the \textit{commutative resolvent algebra}, defined in \cite{vN19}. 

In \cite{BF,vN19,vN21,vNS20}, the commutative resolvent algebra is studied because it is the classical counterpart of the resolvent algebra, a quantum observable algebra that was introduced in \cite{BG07,BG} for the purpose of (nonrelativistic) algebraic quantum field theory. 
This establishes a connection between machine learning and quantum algebra that seems unexplored so far, and for instance different from standard approaches to quantum neural networks \cite{SSP}. A useful application of the noncompact uniform approximation theorem to mathematical quantum physics will be demonstrated in a separate paper \cite{BvN}.

\subsection[The case phi(-inf)=/=phi(inf)]{The case $\varphi(-\infty)\neq\varphi(\infty)$}

Our final main theorem expresses the space of approximable functions in the case of $\varphi(-\infty)\neq\varphi(\infty)$ and gives novel insight into the approximation capability and limitations of neural networks.

When using two or more hidden layers, the space of approximable functions equals the closed span of products (with arbitrarily many factors) of sigmoids composed with one-dimensional projections. 
A way to visualise these products is as the wedge-shaped functions appearing in Figure \ref{fig:two wedge functions} and Definition \ref{defn:wedge function}, related to Voronoi diagrams \cite{MPCB} and tropical geometry \cite{MCT,ZNL}, and familiar to anyone who has ever visualised the approximation behaviour of neural networks in cases were there is a sufficiently complicated structure in the data. 
Indeed, when a neural network is prioritising the fitting of a small-scale structure, at a slightly larger scale one can often see the wedge functions of Definition \ref{defn:wedge function} appearing. See, for example, Figure \ref{fig:exampleANN}. In fact, the rigidity of these wedge functions can prevent the neural network from converging locally if there are not enough nodes or there is not enough time.
Thus, although the mathematical novelty of this paper resides at the `infinitely large' scale, the proof in Section \ref{sct:sigmoids} offers an insightful perspective on the appearance of wedge shapes in general.


\begin{figure}
\begin{center}
\image{0.25}{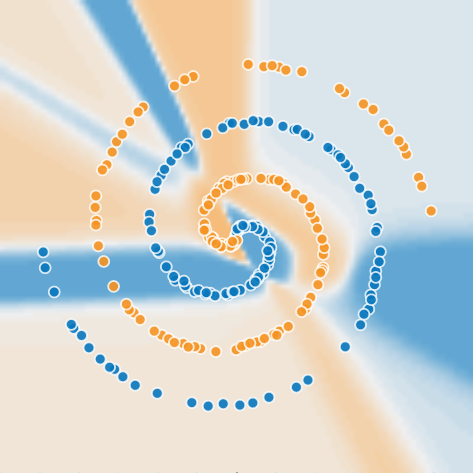}
\caption{Example of a neural network in which wedge functions (cf. Definition \ref{defn:wedge function} and Figure \ref{fig:two wedge functions}) are clearly visible in the contour plot. The network has been given insufficient nodes/layers/time to fit the data at all relevant scales, and has only succeeded on the small scale. At a slightly larger scale the wedge functions already become apparent, and this paper proves that this behaviour is in fact unavoidable at sufficiently large scale. This image was produced using \url{matlabsolutions.com/visualize-neural-network/neural-network.html}.}
\label{fig:exampleANN}
\end{center}
\end{figure}


Opposite to the earlier case, in the present case ($\varphi(-\infty)\neq\varphi(\infty)$) there are two-layer neural networks which cannot be approximated by one-layer neural networks.
We shall give a class of examples of such functions, including quite simple ones. 

\section{Notation and summary of main results}
We let $\N=\{1,2,\ldots\}$. We work over the field $\R$. For any $n\in\N$, we denote by $C(\R^n)$, $C_\textnormal{b}(\R^n)$, $C_0(\R^n)$, and $C_\textnormal{c}(\R^n)$ respectively the continuous functions from $\R^n$ to $\R$, the bounded ones, the ones vanishing at infinity (i.e., $\lim_{\|x\|\to\infty}f(x)=0$), and the compactly supported ones. The support of a function $f$ is denoted by $\supp f$.
By $\overline{S}$ we denote the uniform closure of a set $S$ of functions $\R^n\to\R$, i.e., the closure in the topology induced by the extended metric obtained from the supremum norm. We denote $\overline{\spn}\,S:=\overline{\spn\, S}$, where $\spn\, S$ is the $\R$-linear span of $S$.
For $a,x\in\R^n$, we denote by $a\cdot x:=\sum_{j=1}^na_jx_j$ the Euclidean inner product, and define functions $p_a:\R^n\to\R$ by $p_a(x):=a\cdot x$. We let $P_V:\R^n\to V$ denote the orthogonal projection onto any linear subspace $V\subseteq\R^n$.

Let $\varphi:\R\to\R$ be a function. We define the space of (feedforward) neural networks with $n$ input nodes, one hidden layer, one output node, and activation function $\varphi$, as the following subspace of the vector space of all functions $\R^n\to\R$:
\begin{align}\label{eq:1 hidden layer}
\MLP{1}{n}{1}:=\spn\,\Big\{x\mapsto \varphi(a\cdot x+b)~\Big|~ a\in\R^n,~b\in\R\Big\}.
\end{align}
The space of networks with $l$ hidden layers can then be defined recursively:\footnote{The biases $b$ are redundant for $l\geq2$.}
\begin{align}\label{eq:l hidden layers}
\MLP{l}{n}{1}:=\spn\left\{x\mapsto \varphi(f(x)+b)\mid f\in\MLP{l-1}{n}{1},~b\in\R\right\}.
\end{align}
The space of networks with an arbitrary amount of hidden layers is denoted by $\MLP{\infty}{n}{1}:=\bigcup_{l=1}^{\infty} \MLP{l}{n}{1}.$  Most results shall be stated for networks with only one output node, because extending these results to $m$ output nodes, i.e., to the spaces $\MLP{l}{n}{m}:=\MLP{l}{n}{1}^{\oplus m}$ and $\MLP{\infty}{n}{m}:=\cup_{l}\MLP{l}{n}{m}$, is straightforward. 

The following theorem summarises our first main result, which is formulated for the largest possible class of activation functions in Theorem \ref{thm:C_0 main}. 
\begin{thm}\label{thm:C_0 in 1 layer summary}
Let $n,l\in\N$, and let $\varphi\in C(\R)$ be nonlinear with
$\lim_{x\to\infty}(\varphi(x)-a_1x-b_1)=0$ and $\lim_{x\to-\infty}(\varphi(x)-a_{2}x-b_{2})=0$ for certain $a_1,b_1,a_2,b_2\in\R$. Then, 
$$C_0(\R^n)\subseteq\overline{\MLP{l}{n}{1}}.$$
\end{thm}

The following theorem summarises our second and third main result, which are written in stronger form as Theorem \ref{thm:commutative resolvent algebra} and Theorem \ref{thm:main2}, combined with \ref{thm:distance}.

\begin{thm}\label{thm:main intro}
Let $n\in\N$, and let $\varphi\in C(\R)$ be nonconstant such that $\varphi(-\infty)=\lim_{x\to-\infty}\varphi(x)$ and $\varphi(\infty)=\lim_{x\to\infty}\varphi(x)$ exist and are finite. \begin{enumerate}
\item If $\varphi(-\infty)=\varphi(\infty)$ then $\overline{\MLP{1}{n}{1}}=\overline{\MLP{\infty}{n}{1}}$ and
\begin{align}\label{eq:CR}
\hspace*{-24pt}\overline{\MLP{\infty}{n}{1}} =\overline{\spn}\left\{x\mapsto g(P(x))\mid
\begin{aligned}
&P:\R^n\to\R^k\text{ linear,}\\
&g\in C_0(\R^k),~k\in\Z_{\geq0}
\end{aligned}
\right\}.
\end{align}
\item If $\varphi(-\infty)\neq\varphi(\infty)$ then $\overline{\MLP{2}{n}{1}}=\overline{\MLP{\infty}{n}{1}}$ and
\begin{align}\label{eq:fS}
\hspace*{-24pt}\overline{\MLP{\infty}{n}{1}}=\overline{\spn}\,\Bigg\{x\mapsto\prod_{j=1}^m\tanh(a_j\cdot x)~\Bigg|~ m\in\Z_{\geq0}, ~a_j\in\R^n\Bigg\}.
\end{align}
If, moreover, $n\geq2$, then $\overline{\MLP{1}{n}{1}}\neq \overline{\MLP{2}{n}{1}}$.
\end{enumerate}
\end{thm}
The sigmoid $\tanh$ is used for explicitness, but can be replaced by any sigmoid of choice, as will be discussed in Section \ref{sct:sigmoids}.


\begin{cor}\label{cor:algebra}
Let $n,m\in\N$ and let $\varphi\in C(\R)$ be such that $\lim_{x\to-\infty}\varphi(x)$ and $\lim_{x\to\infty}\varphi(x)$ exist and are finite. Then the vector space $\overline{\MLP{\infty}{n}{1}}$ is an algebra. Equivalently, pointwise products of neural networks are uniformly approximable by neural networks.
\end{cor}
\begin{proof}
If $\varphi$ is constant, $\overline{\MLP{\infty}{n}{1}}$ consists of constant functions, so the statement holds. 

Otherwise, Theorem \ref{thm:main intro} allows us to consider two cases. For case 1 ($\varphi(-\infty)=\varphi(\infty)$), we will now show that the right-hand side of \eqref{eq:CR} is an algebra. For $i=1,2$ we fix $k_i\in\Z_{\geq0}$, $g_i\in C_0(\R^{k_i})$ and linear maps $P_i:\R^n\to\R^{k_i}$. We may always write 
$$P_i(x)=(a_{i,1}\cdot x,\ldots,a_{i,k_i}\cdot x)\qquad (x\in\R^n,~i=1,2),$$
for vectors $a_{i,1},\ldots,a_{i,k_i}\in\R^n$.
We define the number $k:=k_1+k_2$, define the linear function $P:\R^n\to\R^k$ by
$$P(x):=(a_{1,1}\cdot x,\ldots,a_{1,k_1}\cdot x,a_{2,1}\cdot x,\ldots,a_{2,k_2}\cdot x)\qquad(x\in\R^n),$$
and define the function $g:\R^k\to\R$ by
$$g(y_1,\ldots,y_k):=g_1(y_1,\ldots,y_{k_1})g_2(y_{k_1+1},\ldots,y_k)\qquad(y\in\R^k).$$
It follows that $g\in C_0(\R^k)$, and the pointwise product of $g_1\circ P_1$ and $g_2\circ P_2$ evaluates to
$$(g_1\circ P_1)\cdot(g_2\circ P_2)=g\circ P.$$
Hence, by \eqref{eq:CR}, $\overline{\MLP{\infty}{n}{1}}$ is an algebra. 

In case 2 ($\varphi(-\infty)\neq\varphi(\infty)$), the explicit characterisation of $\overline{\MLP{\infty}{n}{1}}$ given by Theorem \ref{thm:main intro} is an algebra by construction.

\end{proof}
We remark that the number $k$ used in the above proof is not necessarily minimal, in the sense that the representation $g\circ P$ in \eqref{eq:CR} is not unique. An in-depth analysis of the algebra in \eqref{eq:CR} is made in \cite{vN19}, including an alternative to the above proof in \cite[Lemma 2.2(i)]{vN19}, cf. Section \ref{sct:CR}.

A particular aspect of Theorem \ref{thm:main intro} is that, for $\varphi(-\infty)\neq\varphi(\infty)$, there exist functions in $\overline{\MLP{2}{n}{1}}$ that are not in $\overline{\MLP{1}{n}{1}}$. In \textsection\ref{sct:1 or 2 layers} we give a class of explicit examples (including a mollified AND function) and thus prove the following stronger statement.
\begin{thm}\label{thm:1 or 2 layers intro}
Let $n\in\N$, $n\geq2$, and let $\varphi\in C(\R)$ be such that $\lim_{x\to-\infty}\varphi(x)$ and $\lim_{x\to\infty}\varphi(x)$ exist and are finite and distinct.
For every $d>0$, there are two-layer networks $f\in\MLP{2}{n}{1}$ with a fixed uniform distance $d$ from the whole collection of one-layer networks, $\MLP{1}{n}{1}$.
\end{thm}
That is, no matter how hard you train the one-layer network $g$ to approximate $f$, no matter the amount of nodes, there will exist an input value $x$ such that $|f(x)-g(x)|>d$.


\section{Approximation of continuous functions vanishing at infinity}\label{sct: Uniform Approximation}

The purpose of this section is to prove Theorem \ref{thm:C_0 main}, which states that $C_0(\R^n)$ can be approximated by neural networks with one hidden layer. This result holds for a slightly larger class of activation functions than mentioned in Theorem \ref{thm:C_0 in 1 layer summary} (in fact, the optimal class), allowing discontinuities and polynomial growth.


However, we shall first prove this approximation theorem in the simpler case that $\varphi\in C_0(\R)\setminus\{0\}$. Furthermore, it will be useful to first consider $n=1$ and $n=2$.

\begin{lem}\label{lem:l=1,n=1}
Let $\varphi\in C_0(\R)\setminus\{0\}$. We have $\overline{\mathcal{N}^1_\varphi(\R)}=C_0(\R)\oplus\R1$.
\end{lem}
\begin{proof}
We observe that any function
$$x\mapsto\varphi(ax+b)\qquad (a,b\in\R)$$
is in $C_0(\R)$ when $a\neq0$. If $a=0$, then the above map is constant, i.e., in $\R 1$. By using \eqref{eq:1 hidden layer} we find that $\mathcal{N}^1_\varphi(\R)\subseteq C_0(\R)\oplus\R1$, and since $C_0(\R)\oplus\R1$ is closed with respect to the supremum norm, we conclude that $\overline{\mathcal{N}^1_\varphi(\R)}\subseteq C_0(\R)\oplus\R1$.

The rest of the proof proceeds exactly as in the proof of \cite[Theorem 1]{Cybenko}, replacing $C([0,1]^n)$ with $C_0(\R)$ and replacing \cite[Lemma 1]{Cybenko} with \cite[Theorem 5]{Hornik} (quite similar to the proof of Proposition \ref{prop:C_0(R2) in 1-layer}). It is however good to note that this strategy naively fails for $C_0(\R^n)$ when $n>1$, as the functions $x\mapsto \varphi(a\cdot x+b)$ are not in $C_0(\R^n)$ when $n>1$.
\end{proof}


For $n>1$, a noncompact uniform approximation theorem requires new ideas not considered by, e.g., \cite{Cybenko,Hornik,KL}.

The core idea in the case $n=2$ is to give meaning to the formal expression
\begin{align}\label{eq:formal integral rotation}
\int_0^{2\pi}g\circ p_{(\cos\theta,\sin\theta)}\,d\theta,
\end{align}
for $p_{(\cos\theta,\sin\theta)}(x,y)=x\cos\theta+y\sin\theta$, and a suitable function $g$ such that \eqref{eq:formal integral rotation} is in $\overline{\MLP{1}{2}{1}}$ and in a way generates $C_0(\R^2)$.
What complicates the proof is that, whatever \eqref{eq:formal integral rotation} means, it
is not a Bochner integral with respect to the supremum norm. Worse yet, the integrand both has inseparable range and is discontinuous, because $\|g\circ p_{(\cos\theta,\sin\theta)}-g\circ p_{(\cos\theta',\sin\theta')}\|_\infty\geq\|g\|_\infty$ for every $\theta\neq\theta'\in[0,\pi)$. The following lemma shows that at least its pointwise interpretation is in $C_0(\R^2)$.
\begin{lem}\label{lem:f is O(R-3)}
Let $g\in C_\textnormal{c}(\R)$,
and define
\begin{align}\label{eq:def f 1}
f(x,y):=\int_0^{2\pi}(g\circ p_{(\cos\theta,\sin\theta)})(x,y)\,d\theta
\end{align}
for all $(x,y)\in\R^2$. Then $f\in C_0(\R^2)$. If moreover $\int_\R g(x)dx=0$, then there exists $C\in\R$ such that
\begin{align}\label{eq:O(R-3)}
|f(x,y)|\leq\frac{C}{1+\|(x,y)\|^3}\qquad ((x,y)\in\R^2).
\end{align}
\end{lem}
\begin{proof}
We rewrite \eqref{eq:def f 1} by noting that, for $(x,y)=R(\cos\varphi,\sin\varphi)$, we have $p_{(\cos\theta,\sin\theta)}(x,y)=R\sin(\theta-\varphi+\pi)$. By a substitution $\theta\mapsto \theta+\varphi-\pi$, 
it follows that
\begin{align}\label{eq:f as g(Rsin)}
f(x,y)=\int_{0}^{2\pi}g(R\sin\theta)\,d\theta,
\end{align}
where $R=\sqrt{x^2+y^2}$. Uniform continuity of $g$ implies that $g(R_n\sin\theta)$ converges to $g(R\sin\theta)$ uniformly in $\theta$ whenever $R_n\to R$, hence proving continuity of $f$. 


Let $a>0$ be such that $\supp g\subseteq[-a,a]$.
If $R>a$ then there exists a $\delta=\delta(R)\in[0,\pi/2)$ such that $\sin(\delta)=a/R$, implying that $R\sin(\theta)\notin\supp g$ whenever $|\sin\theta|>\sin\delta$. Using this $\delta$ and subsequently making the substitution $u=\sin\theta$, we find
\begin{align*}
&f(x,y)\\
&=\int_0^\delta g(R\sin\theta)d\theta+\int_{\pi-\delta}^{\pi+\delta} g(R\sin\theta)d\theta+\int_{2\pi-\delta}^{2\pi} g(R\sin\theta)d\theta\\
&=\int_{-\delta}^\delta (g(R\sin\theta)+g(R\sin(\theta+\pi))d\theta\\
&=\int_{-a/R}^{a/R}(g(Ru)+g(-Ru))\frac{1}{\sqrt{1-u^2}}du\\
&=\int_{-a/R}^{a/R}g(Ru)\frac{2}{\sqrt{1-u^2}}du.
\end{align*}
As $\frac{2}{\sqrt{1-u^2}}=2+\mathcal{O}(u^2)$ for $u\to0$, there exists a $C>0$ such that, for large enough $R$, we have $|\frac{2}{\sqrt{1-u^2}}-2|\leq Cu^2$ for all $u\in[-a/R,a/R]$.
We obtain
\begin{align*}
\left|f(x,y)-2\int_{-a/R}^{a/R}g(Ru)du\right|\leq \int_{-a/R}^{a/R}|g(Ru)|Cu^2du,
\end{align*}
which by substitution $u\mapsto R^{-1}u$ becomes
\begin{align*}
\left|f(x,y)-2R^{-1}\int_{-a}^{a}g(u)du\right|\leq R^{-3}\int_{-a}^{a}|g(u)|Cu^2du.
\end{align*}
Hence $f(x,y)=\mathcal O(R^{-1})=\mathcal O(\|(x,y)\|^{-1})$, in particular $f\in C_0(\R^2)$.
If moreover $\int g(u)du=0$, then for large enough $R=\|(x,y)\|$ we obtain
\begin{align}\label{eq:bound that can be sharpened}
|f(x,y)|\leq R^{-3}\int_{-a}^a|g(u)|Cu^2\,du,
\end{align}
which implies the lemma.
\end{proof}

The following result shows that, although \eqref{eq:formal integral rotation} is not a Bochner integral, its pointwise interpretation is a limit of a particular sequence of Riemann sums as depicted in Figure \ref{fig:uniform approximation 1-layer increasing width}.
\begin{figure}
\begin{center}
\begin{subfigure}{.15\textwidth}
\scalebox{0.6}{
\begin{tikzpicture}
\path[shade, left color=teal, right color=teal,opacity=0.5] (-0.3,-2) -- (0.3,-2) -- (0.3,2) -- (-0.3,2) ;
\path[shade, left color=teal, right color=teal,opacity=0.5] (-2,-0.3) -- (-2,0.3) -- (2,0.3) -- (2,-0.3) ;
\end{tikzpicture}}
\end{subfigure}
\begin{subfigure}{.15\textwidth}
\scalebox{0.6}{
\begin{tikzpicture}
\path[shade, left color=teal, right color=teal,opacity=0.25] (-0.3,-2) -- (0.3,-2) -- (0.3,2) -- (-0.3,2) ;
\path[shade, left color=teal, right color=teal,opacity=0.25] (-2,-0.3) -- (-2,0.3) -- (2,0.3) -- (2,-0.3) ;
\path[shade, left color=teal, right color=teal,opacity=0.25] (-2,-2) -- (-2,-1.5757) -- (1.5757,2) -- (2,2) -- (2,1.5757) -- (-1.5757,-2) ;
\path[shade, left color=teal, right color=teal,opacity=0.25] (-2,2) -- (-2,1.5757) -- (1.5757,-2) -- (2,-2) -- (2,-1.5757) -- (-1.5757,2) ;
\end{tikzpicture}}
\end{subfigure}
\begin{subfigure}{.15\textwidth}
\scalebox{0.6}{
\begin{tikzpicture}
\path[shade, left color=teal, right color=teal,opacity=0.125] (-0.3,-2) -- (0.3,-2) -- (0.3,2) -- (-0.3,2) ;
\path[shade, left color=teal, right color=teal,opacity=0.125] (-2,-0.3) -- (-2,0.3) -- (2,0.3) -- (2,-0.3) ;
\path[shade, left color=teal, right color=teal,opacity=0.125] (-2,-2) -- (-2,-1.5757) -- (1.5757,2) -- (2,2) -- (2,1.5757) -- (-1.5757,-2) ;
\path[shade, left color=teal, right color=teal,opacity=0.125] (-2,2) -- (-2,1.5757) -- (1.5757,-2) -- (2,-2) -- (2,-1.5757) -- (-1.5757,2) ;
\path[shade, left color=teal, right color=teal,opacity=0.125] (-1.153144,2) --(-0.5037,2) -- (1.153144,-2) -- (0.5037,-2);
\path[shade, left color=teal, right color=teal,opacity=0.125] (1.153144,2) --(0.5037,2) -- (-1.153144,-2) -- (-0.5037,-2);
\path[shade, left color=teal, right color=teal,opacity=0.125] (2,-1.153144) --(2,-0.5037) -- (-2,1.153144) -- (-2,0.5037);
\path[shade, left color=teal, right color=teal,opacity=0.125] (-2,-1.153144) --(-2,-0.5037) -- (2,1.153144) -- (2,0.5037);
\end{tikzpicture}}
\end{subfigure}
\caption{First three elements of a sequence of 1-layer neural networks uniformly approximating a function in $C_0(\R^2)$. Cf. Lemma \ref{lem:Riemann sums}. To increase the locality of the limit function, the ridge functions $g\circ p_a$ need to satisfy $\int g(x)dx=0$, unlike what is shown in the picture. Note that $L^p$-convergence is out of the question, as each element of the sequence has infinite integral norm, cf. \cite[Section 7]{Pinkus}.}
\label{fig:uniform approximation 1-layer increasing width}
\end{center}
\end{figure}
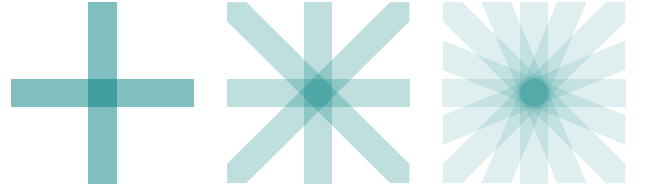
Hence, besides being in $C_0(\R^2)$, the function $f$ of Lemma \ref{lem:f is O(R-3)} is also an element of $\overline{\MLP{1}{2}{1}}$ for any $\varphi\in C_0(\R)\setminus\{0\}$.
\begin{lem}\label{lem:Riemann sums}
Let $g\in C_\textnormal{c}(\R)$ be Lipschitz continuous. We define
\begin{align}\label{eq:def f}
f(x,y):=\int_0^{2\pi}(g\circ p_{(\cos\theta,\sin\theta)})(x,y)\,d\theta
\end{align}
for all $(x,y)\in\R^2$ and
$$f_N:=\frac{2\pi}{N}\sum_{k=0}^{N-1}g\circ p_{(\cos\frac{2\pi k}{N},\sin\frac{2\pi k}{N})},$$
for all $N\in\N$. We have $f_N\in\overline{\MLP{1}{2}{1}}$ for all $\varphi\in C_0(\R)\setminus\{0\}$. Moreover, the sequence $(f_{2^m})_{m\geq1}$ converges uniformly to $f$.
\end{lem}
\begin{proof}
The fact that $f_N\in\overline{\MLP{1}{2}{1}}$ follows by noting that Lemma \ref{lem:l=1,n=1} implies that $g\in C_0(\R)\subseteq \overline{\mathcal{N}^1_\varphi(\R)}$ and that the map $f\mapsto f\circ p_a$ is linear and bounded (with respect to the uniform norm) and sends $\mathcal{N}^1_\varphi(\R)$ into $\MLP{1}{2}{1}$, and hence sends $g\in\overline{\mathcal{N}^1_\varphi(\R)}$ to $g\circ p_a\in\overline{\MLP{1}{2}{1}}$.

For all $(x,y)\in\R^2$, we define the number
$$\Phi_N(x,y):=\#\left\{k\in\{0,\ldots,N-1\}\mid p_{(\cos\frac{2\pi k}{N},\sin\frac{2\pi k}{N})}(x,y)\in\supp g\right\},$$
which leads to the bound
\begin{align}\label{eq:bound f_N}
|f_N(x,y)|\leq\frac{2\pi\supnorm{g}}{N}\Phi_N(x,y).
\end{align}
If $R[\theta]:\R^2\to\R^2$ denotes the rotation by an angle $\theta$, we have
\begin{align}\label{eq:Phi characterisation}
\Phi_N(x,y)=\sum_{k=0}^{N-1}\Phi_1\big(R[\tfrac{2\pi k}{N}](x,y)\big).
\end{align}
By using \eqref{eq:Phi characterisation} twice, we obtain, for every $M,p\in\N$,
\begin{align}
\Phi_{2^{M+p}}(x,y)&=\sum_{k=0}^{2^{M+p}-1}\Phi_{1}\big(R[\tfrac{2\pi k}{2^{M+p}}](x,y)\big)\nonumber\\
&=\sum_{j=0}^{2^M-1}\sum_{r=0}^{2^p-1}\Phi_{1}\big(R[\tfrac{2\pi (2^pj+r)}{2^{M+p}}](x,y)\big)\nonumber\\
&=\sum_{r=0}^{2^p-1}\sum_{j=0}^{2^M-1}\Phi_{1}\big(R[\tfrac{2\pi j}{2^M}]\big(R[\tfrac{2\pi r}{2^{M+p}}](x,y)\big)\big)\nonumber\\
&=\sum_{r=0}^{2^p-1}\Phi_{2^{M}}\big(R[\tfrac{2\pi r}{2^{M+p}}](x,y)\big).\label{eq:Phi dividing by 2^p}
\end{align}
Moreover, for all $M\in\N$, the vectors $a_k:=(\cos\frac{2\pi k}{2^M},\sin\frac{2\pi k}{2^M})$ are pairwise linearly independent for $k=0,\ldots,2^{M-1}-1$ (and likewise for $k=2^{M-1},\ldots,2^M-1$). Any such linearly independent pair $a_k,a_{k'}$ forms a basis of $\R^2$ inducing a norm that is equivalent to the Euclidean norm, hence inducing a number $C_{kk'}>0$ such that
$$\|(x,y)\|\leq C_{kk'}(|a_k\cdot (x,y)|+|a_{k'}\cdot (x,y)|)\qquad (x,y\in\R).$$
We obtain discs $B_{r_{kk'}}(0)\subseteq\R^2$ around 0 with radii $r_{kk'}$ such that $(x,y)\notin B_{r_{kk'}}(0)$ implies that $p_{a_k}(x,y)\notin\supp g$ or $p_{a_{k'}}(x,y)\notin\supp g$, and hence we obtain a radius $r=r(M)>0$ such that for all $(x,y)\notin B_r(0)$ there is at most one $k\in\{0,\ldots,2^{M-1}-1\}$ such that $p_{a_k}(x,y)\in\supp g$, and similarly for $k\in\{2^{M-1},\ldots,2^M-1\}$. It follows that
$$\Phi_{2^M}(x,y)\leq 2\qquad ((x,y)\in\R^2\setminus B_r(0)).$$
Because the disc $B_r(0)$ is rotation invariant, it follows from \eqref{eq:Phi dividing by 2^p} that $\Phi_{2^{M+p}}(x,y)\leq 2^p\cdot 2$ for all $p\in\N$ and $(x,y)$ outside $B_r(0).$
By \eqref{eq:bound f_N}, we conclude that for all $M\in\N$ there exists an $r>0$ such that 
\begin{align*}
|f_{2^{M+p}}(x,y)|\leq \frac{4\pi\supnorm{g}}{2^M}\qquad(p\in\N,(x,y)\in\R^2\setminus B_r(0)).
\end{align*}
As $f\in C_0(\R^2)$ by Lemma \ref{lem:f is O(R-3)}, it follows that for every $\epsilon>0$ there exists an $M\in\N$ and an $r>0$ such that, for all $m\geq M$ we have
\begin{align}\label{eq:bound outside ball}
\sup_{(x,y)\in\R^2\setminus B_r(0)}|f_{2^{m}}(x,y)-f(x,y)|<\epsilon.
\end{align}

Restricting to the compact $\overline{B_r(0)}$, the function 
$$[0,2\pi)\to C_\textnormal{b}(\overline{B_r(0)}),\qquad\theta\mapsto (g\circ p_{(\cos\theta,\sin\theta)})\!\!\upharpoonright_{\overline{B_r(0)}}$$
is $\supnorm{\cdot}$-continuous and separable valued, and hence Bochner integrable. We define simple functions $s_N:[0,2\pi)\to C_\textnormal{b}(\overline{B_r(0)})$ by 
$$s_N:=\sum_{k=0}^{N-1}(g\circ p_{(\cos\frac{2\pi k}{N},\sin\frac{2\pi k}{N})})\!\!\upharpoonright_{\overline{B_r(0)}}\,\,\cdot\,1_{\big[\frac{2\pi k}N,\frac{2\pi (k+1)}{N}\big)},$$
where $1$ is the indicator function. We obtain $\int_{[0,2\pi)} s_N=f_N\!\!\upharpoonright_{\overline{B_r(0)}}\,$ and
\begin{align}\label{eq:supnorm theta}
&\|s_N(\theta)-(g\circ p_{(\cos\theta,\sin\theta)	})\!\!\upharpoonright_{\overline{B_r(0)}}\|_\infty\nonumber\\
&\leq cr\Big(\big|\cos\tfrac{2\pi k_{N,\theta}}{N}-\cos\theta\big|+\big|\sin\tfrac{2\pi k_{N,\theta}}{N}-\sin\theta\big|\Big),
\end{align}
where $c$ is the Lipschitz constant of $g$ and $k_{N,\theta}$ is the unique number such that $\theta\in\big[\frac{2\pi k_{N,\theta}}{N},\frac{2\pi (k_{N,\theta}+1)}{N}\big)$. As the latter interval has length $\frac{2\pi}{N}$ and $\cos$ has a Lipschitz constant of 1, we can bound \eqref{eq:supnorm theta} by $2cr\frac{2\pi}{N}$, which is independent of $\theta$. Hence $\eqref{eq:supnorm theta}$ converges to 0 uniformly in $\theta\in[0,2\pi)$. Therefore, by using that the Bochner integral commutes with the bounded linear map $f\mapsto f\!\!\upharpoonright_{\overline{B_r(0)}}\,$, and subsequently applying the definition of the Bochner integral, we obtain
$$f\!\!\upharpoonright_{\overline{B_r(0)}}\,~=\int_0^{2\pi}(g\circ p_{(\cos\theta,\sin\theta)})\!\upharpoonright_{\overline{B_r(0)}}\,d\theta=\lim_{N\to\infty}f_N\!\!\upharpoonright_{\overline{B_r(0)}}$$
uniformly.
Combining this with \eqref{eq:bound outside ball}, we obtain the lemma.
\end{proof}

The importance of the following lemma can be appreciated by noting that, for all $f\in\MLP{1}{2}{1}$, the corresponding $\psi$ is either constant or undefined.
\begin{lem}\label{lem:nonzero}
Let $\varphi\in C_0(\R)$, and let $g\in C_\textnormal{c}(\R)$ be Lipschitz continuous and satisfy $\int g(x)=0$ and 
$\int g(x)x^2\neq0$.
Let $f$ be defined by \eqref{eq:def f}. Then $f\in C_0(\R^2)\cap\overline{\MLP{1}{2}{1}}$ and $f(x,y)=\mathcal O(\|(x,y)\|^{-3})$ for $\|(x,y)\|\to\infty$. Furthermore, the integral
$$\psi(x)=\int_\R f(x,y)\,dy$$
defines a function $\psi\in L^1(\R)\cap C_0(\R)$, in fact, $|\psi(x)|\leq C(1+|x|)^{-2}$ for $|x|\to\infty$. Lastly, $\psi$ is nonzero.
\end{lem}
\begin{proof}
The statements $f\in C_0(\R^2)\cap\overline{\MLP{1}{2}{1}}$ and $f(x,y)=\mathcal O(\| (x,y)\| ^{-3})$ follow from Lemmas \ref{lem:Riemann sums} and \ref{lem:f is O(R-3)}. 
We moreover deduce that
\begin{align*}
|\psi(x)|&\leq\int_\R|f(x,y)|\,dy\leq C'\int_\R\frac{1}{(1+|x|+|y|)^3}dy\\
&=2C'\int_0^\infty\frac{1}{(1+|x|+y)^3}dy=2C'\int_{1+|x|}^\infty z^{-3}\,dz\\
&=C'(1+|x|)^{-2}.
\end{align*}
Continuity of $\psi$ follows from \eqref{eq:O(R-3)} and the dominated convergence theorem, and hence $\psi\in L^1(\R)\cap C_0(\R)$.
For the last statement, we note that the proof of \eqref{eq:bound that can be sharpened} can be sharpened by using $\frac{2}{\sqrt{1-u^2}}=2+u^2+\mathcal{O}(u^4)$. Again denoting $R=\sqrt{x^2+y^2}$, we obtain
$$f(x,y)=R^{-3}\int_{-a}^ag(u)u^2\,du+\mathcal{O}(R^{-5}).$$
Without loss of generality, $\int g(u)u^2>0$. We have $R\geq |x|$, so if $|x|$ is large enough, $f(x,y)>0$ for all $y\in\R$. Hence $\psi(x)>0$ for such $x$.
\end{proof}

\begin{prop}\label{prop:C_0(R2) in 1-layer}
For all $\varphi\in C_0(\R)\setminus\{0\}$ we have $C_0(\R^2)\subseteq\overline{\MLP{1}{2}{1}}$.
\end{prop}
\begin{proof}
With the intention of finding a contradiction, we assume that $\overline{\MLP{1}{2}{1}}\cap C_0(\R^2)\neq C_0(\R^2)$. By the Hahn--Banach theorem, we obtain a continuous linear map $L:C_0(\R^2)\to\R$ such that $L\left(\overline{\MLP{1}{2}{1}}\cap C_0(\R^2)\right)=\{0\}$ and $L\neq 0$. By the Riesz--Markov--Kakutani theorem, there exists a finite signed Borel measure $\mu$ on $\R^2$ such that
$$L(f)=\int_{\R^2} f(x,y)d\mu(x,y)\qquad(f\in C_0(\R^2)).$$
As $L\neq0$, we obtain $\mu\neq0$. Let $g,f,\psi$ be as in Lemma \ref{lem:nonzero}.

For all $a\in\R^2$ and $b,y\in\R$, let $a_\perp\in\R^2$ be a unit vector such that $a\cdot a_\perp=0$. Define $\tilde f_{a,a_\perp,b,y}(v):=f(a\cdot v+b,a_\perp\cdot v+y)$ for $v\in\R^2$. If $h\in\mathcal{N}_\varphi^1(\R^n)$, then $x\mapsto h(Ax+b)$ belongs to $\mathcal{N}_\varphi^1(\R^n)$ as well, for any matrix $A\in\R^{n\times n}$ and any $b\in\R^n$; this follows directly from the definition. This easily extends to the closure, and hence $\tilde f_{a,a_\perp,b,y}\in\overline{\MLP{1}{2}{1}}\cap C_0(\R^2)$. By a substitution, and \eqref{eq:O(R-3)}, we find
\begin{align}\label{eq:discriminatory2}
\nonumber\int_{\R^2}\psi(a\cdot v+b)\,d\mu(v)&=\int_{\R^2}\int_\R f(a\cdot v+b,y)\,dy\,d\mu(v)\\
\nonumber&=\int_\R\int_{\R^2} f(a\cdot v+b,a_\perp\cdot v+ y)\,d\mu(v)\,dy\\
&=\int_\R L(\tilde f_{a,a_\perp,b,y})\,dy=0.
\end{align}
But, since $\psi$ is bounded and nonconstant, \cite[Theorem 5]{Hornik} implies that there exists no nonzero finite measure $\mu$ such that \eqref{eq:discriminatory2} holds for all $a\in\R^2$ and $b\in\R$. We obtain a contradiction, which implies the lemma.
\end{proof}

We now move to higher dimensions, and obtain a noncompact uniform approximation theorem in the case that $\varphi\in C_0(\R)\setminus\{0\}$.
\begin{prop}\label{prop:C_0(R^n) in 1-layer}
Let $\varphi\in C_0(\R)\setminus\{0\}$, and let $n\in\N$. Any function in $C_0(\R^n)$ is uniformly approximable on $\R^n$ by functions of the form
$$x\mapsto \sum_{j=1}^kc_j\,\varphi(a_j\cdot x+b_j),$$
for $k\in\N$, $b_j,c_j\in\R$, and $a_j\in\R^n$. In fact, for any $l\in\N$, $$C_0(\R^n)\subseteq\overline{\MLP{l}{n}{1}}.$$
\end{prop}
\begin{proof}
We prove the first statement of the proposition by induction on $n$, and notice that $n=1$ is Lemma \ref{lem:l=1,n=1} and $n=2$ is Proposition \ref{prop:C_0(R2) in 1-layer}. For $n>2$, we use $C_0(\R^n)=\overline{C_0(\R)\otimes C_0(\R^{n-1})}$, which follows since $C_0(\R)\otimes C_0(\R^{n-1})$ is a subalgebra of $C_0(\R^n)$ that vanishes nowhere and separates points, and hence its closure equals $C_0(\R^n)$ by the locally compact version of the Stone--Weierstrass theorem. Therefore, in order to prove the first statement of the proposition it suffices to show that any function
$$(x_1,\ldots,x_n)\mapsto f_1(x_1)f_2(x_2,\ldots,x_n)$$
can be uniformly approximated by one-layer networks, for $f_1\in C_0(\R)$ and $f_2\in C_0(\R^{n-1})$. By induction hypothesis, $f_2$ can be uniformly approximated by linear combinations of functions of the form $\varphi\circ p$, for affine maps $p:\R^{n-1}\to\R$, so we may assume without loss of generality that $f_2=\varphi\circ p$ for a fixed affine map $p:\R^{n-1}\to\R$. It then suffices to approximate the function
$$(x_1,\ldots,x_n)\mapsto f_1(x_1)(\varphi\circ p)(x_2,\ldots,x_n),$$
which equals $(f_1\otimes\varphi)\circ P$ for the affine map $P:\R^n\to\R^2$ given by $P(x_1,\ldots,x_n)=(x_1,p(x_2,\ldots,x_n))$.
The function $f_1\otimes\varphi\in C_0(\R^2)$ can be approximated by one-layer networks by using Proposition \ref{prop:C_0(R2) in 1-layer}, and one-layer networks composed with an affine map $P$ are still one-layer networks, yielding the first statement of the proposition.

Let $f\in \overline{\MLP{l-1}{n}{1}}$ and let $g\in C_0(\R)$ equal the identity on the range of $f$. By Lemma \ref{lem:l=1,n=1}, $g$ is uniformly approximated by one-layer networks $(g_m)_{m\geq1}\subseteq\mathcal N_\varphi^1(\R)$. We write $g\circ f-g_m\circ f_m=(g\circ f-g\circ f_m)+(g\circ f_m-g_m\circ f_m)$ for a sequence $(f_m)_{m\geq1}\subseteq\MLP{l-1}{n}{1}$ converging uniformly to $f$. By the uniform continuity of $g$ and an $\epsilon/2$ argument, we obtain $g\circ f=\lim_m g_m\circ f_m\in\overline{\MLP{l}{n}{1}}$. Hence, $f=g\circ f\in\overline{\MLP{l}{n}{1}}$. A straightforward induction argument yields the proposition.
\end{proof}

As a consequence, we obtain a noncompact uniform approximation theorem for all activation functions in the optimal class, which -- as argued below -- has Theorem \ref{thm:C_0 in 1 layer summary} as a special case.
\begin{thm}\label{thm:C_0 main}
A function $\varphi:\R\to\R$ satisfies 
\begin{align}\label{eq:condition}
\overline{\mathcal{N}^1_\varphi(\R)}\cap C_0(\R)\neq\{0\}
\end{align}
if and only if for all $n,l\in\N$ we have
$$C_0(\R^n)\subseteq \overline{\MLP{l}{n}{1}}.$$
\end{thm}
\begin{proof}
The `if' direction follows by taking $l=n=1$. 

For the converse direction, fix any $\chi\in \overline{\mathcal{N}^1_\varphi(\R)}\cap C_0(\R)\setminus\{0\}$. By Proposition \ref{prop:C_0(R^n) in 1-layer}, we obtain
$C_0(\R^n)\subseteq \overline{\mathcal{N}_{\chi}^l(\R^n)}.$ It will therefore suffice to show that $\overline{\mathcal{N}_{\chi}^l(\R^n)}\subseteq \overline{\mathcal{N}_{\varphi}^l(\R^n)}$, which we will do by induction on $l$.
By defining the space of 0-layer neural networks as 
$$\mathcal{N}_\phi^0(\R^n)=\{p:\R^n\to\R~:~p\text{ linear}\},$$
independently of $\phi$, the induction basis $\overline{\mathcal{N}_{\chi}^0(\R^n)}\subseteq \overline{\mathcal{N}_{\varphi}^0(\R^n)}$ follows trivially, and, \eqref{eq:l hidden layers} holds also for $l=1$ and can hence be applied in the induction step.
Let $l\in\N$ and assume as an induction hypothesis that $\overline{\mathcal N_\chi^{l-1}(\R^n)}\subseteq\overline{\mathcal N_\varphi^{l-1}(\R^n)}$.
We will show that $\mathcal{N}^l_{\chi}(\R^n)\subseteq \overline{\mathcal{N}_{\varphi}^l(\R^n)}$ by taking $g\in\mathcal{N}^l_{\chi}(\R^n)$ arbitrary. There exist $J\in\N$, $f_j\in\overline{\mathcal{N}^{l-1}_{\varphi}(\R^n)}$, and $b_j,c_j\in\R$ ($j=1,\ldots,J$) such that
\begin{align*}
g(x)&=\sum_{j=1}^Jc_j\,\chi(f_j(x)+b_j).
\end{align*}
We let $(\chi_m)_{m\geq1}\subseteq\mathcal{N}^1_\varphi(\R)$ be such that $\|\chi_m-\chi\|_\infty\to0$, and we let $(f_{j,m})_{m\geq1}\subseteq\MLP{l-1}{n}{1}$ be such that $\|f_{j,m}-f_j\|_\infty\to0$ as $m\to\infty$. We define the functions 
\begin{align*}
g_m(x)&:=\sum_{j=1}^Jc_j\,\chi_m(f_{j,m}(x)+b_j),
\end{align*}
and conclude -- from the uniform continuity of $\chi\in C_0(\R)$ and an $\epsilon/2$-argument -- that $\|g_m-g\|_\infty\to0$.
We may assume that $\chi_m$ is of the form
$$\chi_m(x)=\sum_{k=1}^{K_m}\tilde c_{k,m}\varphi(\tilde a_{k,m}x+\tilde b_{k,m}),$$
for $K_m\in\N$, $\tilde a_{k,m},\tilde b_{k,m},\tilde c_{k,m}\in\R$, which implies
\begin{align*}
g_m(x)&=\sum_{j=1}^J\sum_{k=1}^{K_m}c_j\tilde c_{k,m}\varphi(\tilde a_{k,m}f_{j,m}(x)+a_j\tilde b_{k,m}+b_j),
\end{align*}
and therefore $g_m\in\MLP{l}{n}{1}$ (conceptually, we have inserted the hidden layer of $\chi_m$ inside each node of the last hidden layer of $g_m$). We deduce that $g\in\overline{\MLP{l}{n}{1}}$, hence, $\mathcal{N}^l_{\chi}(\R^n)\subseteq \overline{\MLP{l}{n}{1}}$, which immediately implies $\overline{\mathcal{N}^l_{\chi}(\R^n)}\subseteq \overline{\MLP{l}{n}{1}}$, the required induction step. We conclude that for all $n,l\in\N$ we have
$$C_0(\R^n)\subseteq \overline{\mathcal N^l_\chi(\R^n)}\subseteq\overline{\MLP{l}{n}{1}},$$
completing the proof.
\end{proof}

We make the case that the assumption \eqref{eq:condition} is satisfied by most activation functions. First of all, we show that this assumption is implied by the assumptions of Theorem \ref{thm:C_0 in 1 layer summary}, which hold for activation functions like for instance ReLU, Leaky ReLU, ELU, GELU, Swish, Softplus, and all sigmoidal or $C_0$ functions.

\begin{proof}[Proof of Theorem \ref{thm:C_0 in 1 layer summary}.]
We define $\chi:=\varphi_1-\varphi\in C_\textnormal{b}(\R)$, where $\varphi_m(x):=\varphi(x-m)$. It follows that $\chi$ has finite limits at $\pm\infty$. 

We now claim that $\chi$ is nonconstant, which we shall prove by contradiction. 

Suppose that $\chi$ is constant, and define
\begin{align}\label{eq:def psi linear}
\psi(x):=\varphi(x)-a_2x-b_2.
\end{align}
Then $\lim_{x\to-\infty}\psi(x)=0$. Furthermore, $\xi(x):=\psi(x-1)-\psi(x)$ satisfies 
$$\xi(x)=\varphi(x-1)-a_2(x-1)-\varphi(x)+a_2x=\chi(x)+a_2.$$
So $\xi$ is constant, and moreover $\xi(x)=\lim_{x\to -\infty}\xi(x)=0$, which implies $\psi(x)=\psi(x-1)$ for all $x\in\R$, i.e. $\psi$ is periodic. Because $\psi$ has a defined limit at $-\infty$, it must be constant. It follows from \eqref{eq:def psi linear} that $\varphi$ is linear, which contradicts our assumptions.

As $\chi$ is nonconstant, we obtain a nontrivial function
$$\chi_1-\chi=\varphi_2-2\varphi_1+\varphi\in \mathcal N_\varphi^1(\R)\cap C_0(\R)\setminus\{0\},$$
so the assumption of Theorem \ref{thm:C_0 main} is satisfied.
\end{proof}
The same argument repeated shows that continuous and nonpolynomial but asymptotically polynomial activation functions (converging to potentially different polynomials at $-\infty$ and $\infty$) satisfy $\overline{\mathcal{N}^1_\varphi(\R)}\cap C_0(\R)\neq\{0\}$, i.e. \eqref{eq:condition}. This offers an alternative perspective on \cite{LLPS} (cf. \cite[Theorem 3.1]{Pinkus}). The noncompact case is more subtle, as exemplified by the fact that the sum of a sigmoid and \textit{any} even function satisfies \eqref{eq:condition}. Another standard argument shows that \eqref{eq:condition} is satisfied by discontinuous activation functions like the binary step function as well.

\section{Bounded activation functions with identical left and right limits}\label{sct:CR}
In this section we precisely characterise the space of uniformly approximable functions in the easiest situation, namely in the case that $\varphi(-\infty)$ and $\varphi(\infty)$ are finite and equal.

We recall the definition of the commutative resolvent algebra from \cite{vN19}:
\begin{defn}
For $n\in\N$, the (real-valued part of the) commutative resolvent algebra on $\R^n$ is given by
\begin{align*}
\Cr(\R^n)=\overline{\spn}\left\{g\circ P_V\mid
\begin{aligned}&V\subseteq\R^n\text{ linear subspace, }\\
&g\in C_0(V)
\end{aligned}\right\},
\end{align*}
where $P_V$ denotes the orthogonal projection onto the linear subspace $V\subseteq\R^n$.
\end{defn}
There exists a slightly different characterisation of $\Cr(\R^n)$.
\begin{lem}
For all $n\in\N$ we have
\begin{align}\label{eq:alt characterisation Cr}
\Cr(\R^n)=\overline{\spn}\left\{g\circ P\mid
\begin{aligned}
&P:\R^n\to\R^k\text{ linear, }\\
&g\in C_0(\R^k),~ k\in\Z_{\geq0}
\end{aligned}
\right\}.
\end{align}
\end{lem}
\begin{proof}
Each function $g\circ P$, for $P:\R^n\to\R^k$ linear and $g\in C_0(\R^k)$, can be written in the form 
\begin{align}\label{eq:different characterisation g circ P}
g\circ P=\tilde{g}\circ P_V
\end{align}
for $V:=(\ker P)^\perp$, $P_V:\R^n\to V$ the corresponding orthogonal projection, and $\tilde g:=g\circ P\!\!\upharpoonright_V\,$. As $P\!\!\upharpoonright_V\,\,:V\to\ran P\subseteq\R^k$ is a linear isomorphism, and $g\!\!\upharpoonright_{\ran P}\,\,\in C_0(\ran P)$, it follows that $\tilde g\in C_0(V)$. This implies $\supseteq$ of \eqref{eq:alt characterisation Cr} and the converse inclusion follows similarly (if not slightly easier).
\end{proof}
The (real part of the) commutative resolvent algebra is a closed subalgebra of $C_\textnormal{b}(\R^n)$, as shown in the proof of Corollary \ref{cor:algebra}, or, alternatively, in \cite[Lemma 2.2]{vN19}. 
Although \cite{vN19} works over the complex numbers, the above remark (over the real numbers) follows immediately by taking the real part.

Closed subalgebras of $C_\textnormal{b}(\R^n)$ relate to deep neural networks in the following way.

\begin{lem}\label{lem:algebra}
Let $\varphi\in C_\textnormal{b}(\R)$ and $n\in\N$. If $A$ is a (uniformly) closed subalgebra of $C_\textnormal{b}(\R^n)$ with $\MLP{1}{n}{1}\subseteq A$, then $\MLP{l}{n}{1}\subseteq A$ for any $l\in\N$.
\end{lem}
\begin{proof}
The claim is trivial if $\varphi=0$. Thus, assume that $\varphi\neq0$ and note that then $\mathcal{N}_\varphi^1(\R^n)\subseteq A$ contains the constant functions.
For $l\geq2$, assume that $\MLP{l-1}{n}{1}\subseteq A$ and let $f\in\MLP{l-1}{n}{1}$ and $b\in\R$. We are to prove that $\varphi\circ(f+b)\in A$. As $f+b$ is a bounded function, the Stone--Weierstrass theorem supplies a sequence of polynomials $(p_k)_{k\geq1}$ converging to $\varphi$ uniformly on the range of $f+b$. Hence $p_k\circ (f+b)$ converges uniformly to $\varphi\circ(f+b)$. Because $A$ is a unital algebra, $f+b\in A$ implies that
$$p_k\circ(f+b)=p_k(f+b)\in A.$$
As $A$ is closed, we obtain $\varphi\circ(f+b)\in A$, which by induction implies that $\MLP{l}{n}{1}\subseteq A$ for every $l\geq 1$.
\end{proof}

We now turn to the case $\varphi(-\infty)=\varphi(\infty)$. Without loss of generality (because we are considering neural networks with biases $b$) we can assume that $\varphi(-\infty)=\varphi(\infty)=0$, i.e., $\varphi\in C_0(\R)$.

\begin{lem}\label{lem:first inclusion}
For any $\varphi\in C_0(\R)$ and $n\in\N$ we have
	$$\MLP{\infty}{n}{1}\subseteq\Cr(\R^n)\,.$$
\end{lem}
\begin{proof}
Any network in $\MLP{1}{n}{1}$ is a linear combination of functions of the form
\begin{align}\label{eq:levee from varphi}
x\mapsto \varphi(a\cdot x+b),
\end{align}
for $a\in\R^n$ and $b\in\R$. By taking $P(x):=a\cdot x$ and $g(y):=\varphi(y+b)$, we find that the function \eqref{eq:levee from varphi} equals $g\circ P\in\Cr(\R^n)$. Hence, $\MLP{1}{n}{1}\subseteq\Cr(\R^n)$. By Lemma \ref{lem:algebra}, this completes the proof.
\end{proof}

An immediate corollary of the above lemma is 
$$\overline{\MLP{\infty}{n}{1}}\subseteq\Cr(\R^n).$$
The following theorem is the main result of this section, and states that the above inclusion is an equality. In fact, equality is already obtained with one hidden layer.

\begin{thm}\label{thm:commutative resolvent algebra}
	For any $n\in\N$ and any $\varphi\in C_0(\R)\setminus\{0\}$ we have 
	$$\overline{\MLP{\infty}{n}{1}}=\overline{\MLP{1}{n}{1}}=\Cr(\R^n).$$
	Regarding systems with $m$ output nodes, we therefore have
		$$\overline{\MLP{\infty}{n}{m}}=\Cr(\R^n)^{\oplus m}\,.$$
\end{thm}
\begin{proof}
Let $g\circ P\in\Cr(\R^n)$ be a function of the form of \eqref{eq:alt characterisation Cr}, namely with $P:\R^n\to\R^k$ a linear map and $g\in C_0(\R^k)$. By Proposition \ref{prop:C_0(R^n) in 1-layer}, $g\in\overline{\spn}\{x\mapsto\varphi (a\cdot x+b)~|~a\in\R^k,b\in\R\}$. Hence, by the continuity of the map $g\mapsto g\circ P$,
\begin{align*}
g\circ P&\in\overline{\spn}\left\{x\mapsto\varphi(a\cdot P(x)+b)\mid a\in\R^k,b\in\R\right\}\\
&\subseteq\overline{\spn}\left\{x\mapsto\varphi(\tilde a\cdot x+b)\mid \tilde a\in\R^n,b\in\R\right\},
\end{align*}
where the inclusion is obtained by noting that $\R^n\to\R,x\mapsto a\cdot P(x)$ is linear, and hence given by $x\mapsto \tilde a\cdot x$ for an $\tilde a\in\R^n$.
\end{proof}


The following corollary will be used in Section \ref{sct:sigmoids}.
\begin{cor}\label{lem:com res alg in sigmoids}
For any $n,l\in\N$ and any nonconstant $\varphi\in C(\R)$ with finite left and right limits we have
$$\Cr(\R^n)\subseteq \overline{\MLP{l}{n}{1}}.$$
\end{cor}
\begin{proof}
If $\varphi_1(x):=\varphi(x-1)$ denotes the shift of $\varphi$, we have $\varphi-\varphi_1\in C_0(\R)\setminus\{0\}$. Furthermore, we see that $\mathcal{N}_{\varphi-\varphi_1}^l(\R^n)\subseteq \MLP{l}{n}{1}$ for any number of hidden layers $l$. Hence, the result follows from Theorem \ref{thm:commutative resolvent algebra}.
\end{proof}

\section{Bounded activation functions with distinct left and right limits}
\label{sct:sigmoids}
For this section, we let $\varphi$ be continuous with finite and distinct left and right limits.
In this case, describing the set of uniformly approximable functions is slightly more involved. 
\begin{defn}
We define
$$\fS(\R):=\left\{f\in C(\R)\mid \lim_{x\to-\infty}f(x)\text{ and }\lim_{x\to\infty}f(x)\text{ exist in $\R$}\right\}\,.$$
More generally, for $n\in\N$, we define
$$\fS(\R^n):=\overline{\spn}\,\Bigg\{\prod_{j=1}^m (g_j\circ p_{a_j})~\Bigg|~ m\in\N, g_j\in \fS(\R), a_j\in\R^n\Bigg\},$$
where we recall that $p_a(x)=a\cdot x$.
\end{defn}
We note that $\fS(\R^n)$ is a closed subalgebra of $C_\textnormal{b}(\R^n)$.
We may give a more explicit characterisation of $\fS(\R^n)$.
\begin{lem}\label{lem:explicit fS}
For all $n\in\N$ we have
$$\fS(\R^n)=\overline{\spn}\,\Bigg\{\prod_{j=1}^m (\tanh\circ p_{a_j})~\Bigg|~ m\in\Z_{\geq0}, a_j\in\R^n\Bigg\}.$$
In the above formula, $\tanh$ can be replaced by any strictly monotonous element of $\fS(\R)$.
\end{lem}
\begin{remark}
The above remarks can be formulated in C*-algebraic language quite concisely. Namely, for any fixed strictly monotonous $\sigma\in\fS(\R)$, $\fS(\R^n)$ is the smallest C*-subalgebra of the real C*-algebra $C_\textnormal{b}(\R^n)$ that contains the functions $1$ and $\sigma\circ p_a$ $(a\in\R^n)$.
\end{remark}
\begin{proof}[Proof of Lemma \ref{lem:explicit fS}]
The inclusion $\supseteq$ follows by taking $g_j=\tanh$. The right-hand side is therefore a subalgebra of $\fS(\R^n)$. Because composition with $p_a$ is a continuous mapping, it thus suffices to show that, for all $g\in\fS(\R)$,
$$g\in\overline{\spn}\left\{\prod_{j=1}^m\sigma\mid m\in\Z_{\geq0}\right\},$$ 
for a fixed strictly monotonous element $\sigma\in\fS(\R)$ such as $\sigma=\tanh$.
The above set contains all limits of all polynomials in $\sigma$, and hence, by Stone--Weierstrass, it contains $f\circ\sigma$ for every continuous function $f\in C(\overline{\ran\sigma})$. It therefore also contains $g=(g\circ\sigma^{-1})\circ\sigma$,
as required. 
\end{proof}

The algebra $\fS(\R)$ is closely related to the space of one-layer neural networks.

\begin{lem}\label{lem:l=1,n=1 sigmoid case}
For any $\varphi\in \fS(\R)$ we have
$\overline{\mathcal{N}^1_\varphi(\R)}\subseteq\fS(\R).$ If moreover $\varphi(-\infty)\neq\varphi(\infty)$, then we have $\overline{\mathcal{N}^1_\varphi(\R)}=\fS(\R)$.
\end{lem}
\begin{proof}
For any $\varphi\in\fS(\R)$, the space $\mathcal{N}^1_\varphi(\R)$ is spanned by functions
$$x\mapsto \varphi(ax+b)\qquad (a,b\in\R),$$
which are also functions in $\fS(\R)$. Since $\fS(\R)$ is a closed linear space, we obtain $\overline{\mathcal{N}^1_\varphi(\R)}\subseteq\fS(\R)$.

If we furthermore assume that $\varphi(-\infty)\neq\varphi(\infty)$, then for every $f\in\fS(\R)$ there exist $a,b\in\R$ such that $f-a\varphi-b\in C_0(\R)$. Hence, by using Theorem \ref{thm:C_0 in 1 layer summary},
$$f\in a\varphi+b+C_0(\R)\subseteq\overline{\mathcal{N}^1_\varphi(\R)}.$$
The combination of both inclusions finishes the proof.
\end{proof}

One of the two desired inclusions is now derived as follows.
\begin{prop}\label{prop:first inclusion sigmoids}
For every $n\in\N$ and every $\varphi\in \fS(\R)$, we have
$$\overline{\MLP{\infty}{n}{1}}\subseteq \fS(\R^n).$$
\end{prop}
\begin{proof}
The space $\MLP{1}{n}{1}$ is spanned by functions of the form $f=g\circ p_a$ for $g\in\mathcal{N}^1_\varphi(\R)$ and $a\in\R^n$. By Lemma \ref{lem:l=1,n=1 sigmoid case}, we have $g\in\fS(\R)$, which implies that $f\in\fS(\R^n)$.
Therefore, $\MLP{1}{n}{1}\subseteq \fS(\R^n)$, and since $\fS(\R^n)$ is a closed subalgebra of $C_\textnormal{b}(\R^n)$, Lemma \ref{lem:algebra} implies the proposition.
\end{proof}

We proceed with the converse inclusion, in order to obtain equality of the spaces $\overline{\MLP{\infty}{n}{1}}$ and $\fS(\R^n)$.


\subsection{Converse inclusion}\label{sct:converse inclusion}


Denote by $\fS_\text{c}(\R)$ the set of functions $f\in \fS(\R)$ such that $f(\R)= [0,1]$ and $f^{-1}((0,1))$ is bounded. 
\begin{lem}\label{lem:density of fSc}
The span of $\fS_\text{c}(\R)$ is dense in $\fS(\R)$.
\end{lem}
\begin{proof}
%
%
%
%
%
%
%
%
Our definitions imply that $C_0(\R)+\spn\fS_\text{c}(\R)\subseteq\fS(\R)$.
Let $f\in\fS(\R)$ be arbitrary. If $f(-\infty)=f(\infty)$ then $f-f(\infty)v\in C_0(\R)$, where $v\in\fS_\text{c}(\R)$ is defined as $v(x):=\min(1,|x|)$. It follows that $f\in C_0(\R)+\spn\fS_\text{c}(\R)$. 

Define $v_1,v_2\in\fS_{\text{c}}(\R)$ by
\[
v_1(x)=\max\{0,\min\{1,x\}\}\qquad\text{and}\qquad v_2(x)=v_1(-x),\]
so that $v_1(\infty)=1=v_2(-\infty)$ and $v_1(-\infty)=0=v_2(\infty)$. Then, setting $g:=f-f(\infty)v_1-f(-\infty)v_2$, we have $g\in C_0(\R)$ and
\[
f=g+f(\infty) v_1+f(-\infty) v_2\in C_0(\R)+\spn\fS_\text{c}(\R).
\]
Therefore,
\begin{align}\label{eq:fS 1}
C_0(\R)+\spn\fS_\text{c}(\R)=\fS(\R).
\end{align}

For an arbitrary $f\in C_\text{c}(\R)$, decompose $f=f_+-f_-$ for the positive and negative parts $f_+,f_-\geq0$ of $f$. If $f_\pm\neq0$, we have $\frac{f_\pm}{\supnorm{f_\pm}}\in\fS_\text{c}(\R)$, which implies $f_\pm\in\spn\fS_\text{c}(\R)$. Hence,
\begin{align}\label{eq:fS 2}
C_\text{c}(\R)\subseteq\spn\fS_\text{c}(\R).
\end{align}

%

By taking closures of \eqref{eq:fS 1} and \eqref{eq:fS 2}, we obtain $\overline{\spn}\fS_\text{c}(\R)=C_0(\R)+\overline{\spn}\fS_\text{c}(\R)=\fS(\R)$, as claimed.
\end{proof}

\begin{defn}\label{defn:wedge function}
	Let $J$ be a finite set. A \textbf{wedge function} with \textbf{support vectors} $\{a_j\}_{j\in J}\subseteq\R^n$ is a function of the form
		$$(g\circ P_V)\prod_{j\in J} (g_j\circ p_{a_j}),$$
	for some linear subspace $V\subseteq \R^n$, $g\in C_\textnormal{c}(V)$ and $g_j\in \fS_\textnormal{c}(\R)$ for each $j\in J$.
\end{defn}

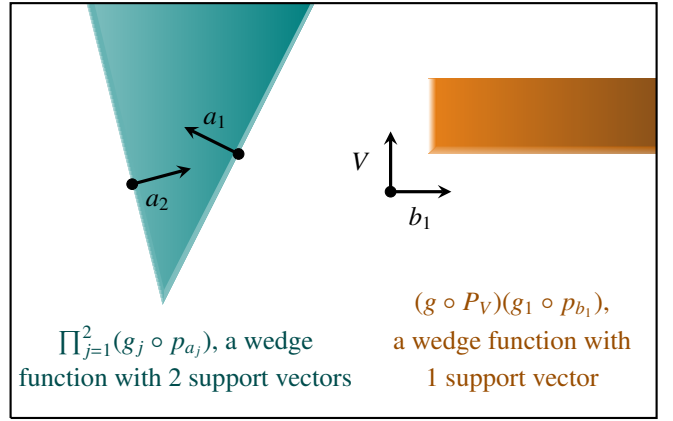
\begin{figure}
\begin{center}
\begin{tikzpicture}
\path[shade, left color=teal!50!white, right color=teal] (3,4) -- (1,0)
node[pos=0.5,minimum size=0.1cm,label={[label distance=0.1cm]95:$a_1$}](x1) {}
 -- (0,4) node[pos=0.4,minimum size=0.1cm,label={[label distance=-2]-10:$a_2$}](x2) {};
\path[shade,left color=teal,right color=white,shading angle=63.4349] (3,4) -- (1,0) -- (1.02,0.2) -- (2.95,4);
\pgfdeclareverticalshading{tealshading}{100bp}{color(0bp)=(teal!60!white); color(49bp)=(teal!60!white); color(51bp)=(white); color(100bp)=(white)}
\path[shading=tealshading,shading angle=104] (1,0) -- (0,4) -- (0.1,4) -- (1.02,0.2);
\filldraw (x1.center) circle (2pt);
\filldraw (x2) circle (2pt);
\draw [-stealth,very thick] (x1.center) -- ++(153.4349:0.8);
\draw [-stealth,very thick] (x2.center) -- ++(15.0362:0.8);
\node at (1.3,-0.5) {\color{teal!60!black}$\prod_{j=1}^2(g_j\circ p_{a_j})$, a wedge};
\node at (1.3,-1) {\color{teal!60!black} function with 2 support vectors};

\path[shade, left color=orange!90!black, right color=orange!50!black,opacity=0.9] (7.5,2) -- (4.5,2) -- (4.5,3) -- (7.5,3);
\node (o) at (4,1.5) {};
\path[shade,top color=orange!85!black,bottom color=white,shading angle=-1,opacity=0.5] (7.5,2) -- (4.5,2) -- (4.6,2.1) -- (7.5,2.1);
\path[shade,left color=white,right color=orange!90!black,shading angle=90,opacity=1] (4.49,2) -- (4.49,3) -- (4.6,3) -- (4.6,2.1);
\filldraw (o) circle (2pt);
\draw [-stealth,very thick] (o.center) -- ++(0:0.8);
\draw [-stealth,very thick] (o.center) -- ++(90:0.8);
\node at (4.4,1.15) {$b_1$};
\node at (3.6,1.9) {$V$};
{\color{orange!60!black}
\node at (5.6,0) {$(g\circ P_V)(g_1\circ p_{b_1})$,};
\node at (5.6,-0.5) {a wedge function with};
\node at (5.6,-1) {1 support vector};
}

\draw[thick] (-1,-1.5) -- (-1,4) -- (7.5,4) -- (7.5,-1.5) -- (-1,-1.5);
\end{tikzpicture}
\caption{Contour plot of two wedge functions on $\R^2$.}
\label{fig:two wedge functions}
\end{center}
\end{figure}

The following proposition is the key to Theorem \ref{thm:main2}.
\begin{prop}\label{prop:wedge functions}
Fix $\varphi\in\fS(\R)$ with $\varphi(-\infty)\neq\varphi(\infty)$ and let $m\in\N$. Any wedge function with $m$ support vectors is in $\overline{\MLP{2}{n}{1}}$. 
\end{prop}


\begin{proof}
The proof proceeds by induction on $m$. The induction basis follows immediately, since wedge functions with no support vectors are in $\overline{\MLP{2}{n}{1}}$ by Corollary \ref{lem:com res alg in sigmoids}. 

Let $J$ be an index set of support vectors $\{a_j\}_{j\in J}$ such that all wedge functions with strictly less support vectors are in $\overline{\MLP{2}{n}{1}}$ -- this is the induction hypothesis. To obtain the induction step, we shall fix a generic wedge function $(g\circ P_V)\prod_{j\in J}(g_j\circ p_{a_j})$ and prove that it is in $\overline{\MLP{2}{n}{1}}$. That is, we fix a linear subspace $V\subseteq\R^n$, functions $g\in C_\textnormal{c}(V)$ and $g_j\in \fS_\text{c}(\R)$ and support vectors $a_j\in\R^n$ for all $j\in J$. Without loss of generality, we may assume that $0\leq g\leq 1$.
We define a function $h\in C_\textnormal{b}(\R)$ by
	$$h(x):=
	\begin{cases}
	0 \quad& \text{if}\quad x\leq0,\\
	x & \text{if}\quad 0\leq x\leq 1,\\
	1 & \text{if}\quad 1\leq x,
	\end{cases}$$
as well as the shifts $h_m(x):=h(x-m)$.
	For every $I\subseteq J$, we recursively define
	\begin{align}
		f_\emptyset :=~& 0\nonumber,\\
		f_I:=~&(g\circ P_V)\prod_{j\in I}(g_j\circ p_{a_j})-h_{\#I}\circ\bigg(g\circ P_V+\sum_{j\in I}(g_j\circ p_{a_j})\bigg)\nonumber\\
		&-\sum_{\substack{H\subseteq I\\ H\neq I}}f_H\prod_{j\in I\setminus H}(g_j\circ p_{a_j})\,.\label{eq:f_K}
	\end{align}
	We note that the two definitions are consistent by taking $I=\emptyset$ and using $g\circ P_V-h_0\circ(g\circ P_V)=0$.
	We also define the space $W_I:=V+\spn\{a_j~|~ j\in I\}$. By induction on $\#I$, and using that $p_{a_j}=p_{a_j}\circ P_{W_I}$ for every $j\in I$, we find that $f_I=f_I\circ P_{W_I}$. Moreover, we claim that $f_I\!\!\upharpoonright_{W_I}\,\,\in C_\textnormal{c}(W_I)$. The latter claim is shown by proving the following statement by induction (i.e., we are using nested induction). Here, $x\in\mathbb R^n$ is fixed.
	\begin{align*}
	\boxed{
		\alpha(I):\qquad \begin{aligned}
		&\text{``If }g(P_V(x))=0\\
		&\text{or }g_{j_0}(a_{j_0}\cdot x)\in \{0,1\}\text{ for some $j_0\in I$,}\\
		&\text{then $f_I(x)=0$."}
		\end{aligned}
	}
	\end{align*}
		
	We prove $\alpha(I)$ by induction on $\#I$. By definition, $\alpha(\emptyset)$ is true. Now suppose $\alpha(H)$ is true for all $H\subsetneq I$. Then $\alpha(I)$ follows from the following three statements:
	\begin{itemize}
		\item If $g(P_V(x))=0$, then $f_H(x)=0$, so \eqref{eq:f_K} gives $f_I(x)=0$.
	
		\item If $g_{j_0}(a_{j_0}\cdot x)=0$ for a certain $j_0\in I$, then the first and second term of \eqref{eq:f_K} drop out, leaving us with
			$$f_I(x)=-\sum_{\substack{H\subseteq I\\ H\neq I}}f_H(x)\prod_{j\in I\setminus H}g_j(a_j\cdot x)\,.$$
		If $j_0\in H$, then we have $f_H(x)=0$, and if $j_0\notin H$, then $\prod_{j\in I\setminus H} g_j(a_j\cdot x)=0$. Hence, in both cases, $f_I(x)=0$.
		
		\item If $g_{j_0}(a_{j_0}\cdot x)=1$ for a certain $j_0\in I$, then
		\begin{align*}
			f_I(x)=& g(P_V(x))\prod_{j\in {I\setminus\{j_0\}}}g_j(a_j\cdot x)\\
			&-h_{\#I-1}\bigg(g(P_V(x))+\sum_{j\in I\setminus\{j_0\}}g_j(a_j\cdot x)\bigg)\\
			&-\sum_{\substack{H\subseteq I\\ H\neq I}} f_H(x)\prod_{j\in (I\setminus\{j_0\})\setminus H}g_j(a_j\cdot x).
		\end{align*}
		For all $H\subsetneq I$ with $j_0\in H$ we have $f_H(x)=0$, so the third term becomes
		\begin{align*}
			-&\sum_{\substack{H\subseteq I\\H\neq I}} f_H(x)\prod_{j\in (I\setminus\{j_0\})\setminus H}g_j(a_j\cdot x)\\
			&=-\sum_{\substack{H\subseteq I\setminus\{j_0\}\\ H\neq I\setminus\{j_0\}}} f_H(x)\prod_{j\in I\setminus H}g_j(a_j\cdot x)-f_{I\setminus\{j_0\}}(x)\,.
		\end{align*}
		We conclude that $f_I(x)=f_{I\setminus\{j_0\}}(x)-f_{I\setminus\{j_0\}}(x)=0$.
	\end{itemize}
	Therefore $\alpha(H)$ is true for every $H\subseteq J$. We will now deduce that $f_H\!\!\upharpoonright_{W_H}~$ has compact support.
	The assertion $\alpha(H)$ shows that for $x\in W_H$ with $f_H(x)\neq0$ we have $P_V(x)\in\supp g$ and $a_j\cdot x\in \overline{g_j^{-1}((0,1))}$ for $j\in H$, which by compactness of $\supp g$ and $\overline{g_j^{-1}((0,1))}$ implies that there exists $R>0$ independent of $x$ such that $\|P_V(x)\|\leq R$ and $|a_j\cdot x|\leq R$. It is not hard to see that
\[
\|x\|_\ast:=\|P_V(x)\|+\sum_{j\in H}|a_j\cdot x|\qquad (x\in W_H),
\]
defines a seminorm, and also positive definiteness is proven as follows. If $\|x\|_\ast=0$, then $x\perp V$ and $x\perp a_j$ for all $j\in H$. Since $x\in W_H=V+\spn\{a_j~|~j\in I\}$, this implies $x\perp x$ and hence $x=0$.
Now, the considerations from above imply that the set $\{x\in W_H~|~f_H(x)\neq 0\}$ is bounded with respect to the norm $\|\cdot\|_\ast$ (hence bounded with respect to any norm on $W_H$), which implies that $f_H\!\!\upharpoonright_{W_H}$ has compact support. We have noted earlier that $f_H=f_H\circ P_{W_H}$, so we conclude that
\begin{align}\label{eq:splitting}
	f_H=f_H\!\!\upharpoonright_{W_H}\circ~ P_{W_H}\quad\text{and}\quad f_H\!\!\upharpoonright_{W_H}~\in C_c(W_H),
\end{align}
in particular, $f_H$ is a wedge function without support vectors.
			
By rearranging \eqref{eq:f_K} in the case $I=J$, we obtain
\begin{align}
	(g\circ P_V)\prod_{j\in J}(g_j\circ p_{a_j})=&\sum_{\emptyset\neq H\subseteq J} f_H\prod_{j\in J\setminus H}(g_j\circ p_{a_j})\nonumber\\
	&+h_{\#J}\circ\Big(g\circ P_V+\sum_{j\in J}(g_j\circ p_{a_j})\Big)\label{eq:wedge functions}\,.
\end{align}
We can summarise \eqref{eq:splitting} and \eqref{eq:wedge functions} by saying that any wedge function can be written as a sum of wedge functions with strictly less support vectors (that are therefore in $\overline{\MLP{2}{n}{1}}$ by the induction hypothesis) plus a remainder term which we denote by
$$r := h_{\#J}\circ \Big(g\circ P_V+\sum_{j\in J}(g_j\circ p_{a_j})\Big)\,.$$

We now show that $r\in \overline{\MLP{2}{n}{1}}$.
Corollary \ref{lem:com res alg in sigmoids} implies that $g\circ P_V\in\Cr(\R^n)\subseteq\overline{\MLP{1}{n}{1}}$. Moreover, Lemma \ref{lem:l=1,n=1 sigmoid case} shows that $g_j\in\fS_\text{c}(\R)\subseteq\overline{\mathcal N_\varphi^1(\R)}$ and hence $g_j\circ p_{a_j}\in\overline{\mathcal N_\varphi^1(\R^n)}$. Therefore, 
$$g\circ P_V+\sum_{j\in J}(g_j\circ p_{a_j})\in\overline{\MLP{1}{n}{1}}.$$ By Lemma \ref{lem:l=1,n=1 sigmoid case}, $h_{\#J}\in\overline{\mathcal{N}^1_\varphi(\R)}$. By uniform continuity of elements in $\mathcal{N}^1_\varphi(\R)$, it follows that $r\in\overline{\MLP{2}{n}{1}}$.

Hence, the function in \eqref{eq:wedge functions} is in $\overline{\MLP{2}{n}{1}}$, which completes the induction step initiated at the beginning of this proof. We conclude that all wedge functions are in $\overline{\MLP{2}{n}{1}}$.
\end{proof}

\begin{thm}\label{thm:main2}
	Let $\varphi\in C(\R)$ be such that $\lim_{x\to-\infty}\varphi(x)$ and $\lim_{x\to\infty}\varphi(x)$ are finite and unequal. We have 
	$$\overline{\MLP{\infty}{n}{1}}=\overline{\MLP{2}{n}{1}}=\fS(\R^n).$$
	Regarding systems with $m$ output nodes, we therefore have
	$$\overline{\MLP{\infty}{n}{m}}=\fS(\R^n)^{\oplus m}.$$
\end{thm}
\begin{proof}
By Proposition \ref{prop:wedge functions} we know that all wedge functions belong to $\overline{\MLP{2}{n}{1}}$. In particular, taking $V=\{0\}$, wedge functions of the form $\prod_{j\in J} (g_j\circ p_{a_j})$ belong to $\overline{\MLP{2}{n}{1}}$. By using Lemma \ref{lem:density of fSc}, this implies that $\fS(\R^n)\subseteq\overline{\MLP{2}{n}{1}}$. Combining the latter inclusion with the inclusion from Proposition \ref{prop:first inclusion sigmoids}, we obtain equality.
\end{proof}

\subsection{Difference between one-layer and two-layer networks}
\label{sct:1 or 2 layers}
Let $\varphi\in\fS(\R)$ with $\varphi(-\infty)\neq\varphi(\infty)$.
Let $h\in\fS(\R)$ be a continuous function satisfying $h(x)=0$ for $x\leq0$ and $h(x)=1$ for $x\geq 1$. Then 
\begin{align}\label{eq:f mollified AND}
f(x,y):=h(x)h(y)
\end{align}
(a mollified AND function) is approximable by two-layer neural networks by Theorem \ref{thm:main2}. However, the following reasoning shows that it is at least a (supremum norm) distance of 1/4 from all one-layer neural networks.
\begin{thm}\label{thm:distance}
Let $\varphi\in\fS(\R)$ with $\varphi(-\infty)\neq\varphi(\infty)$, and define $f\in\overline{\mathcal{N}_\varphi^2(\R^2)}$ by \eqref{eq:f mollified AND}.
Then 
\[
\|f-g\|_\infty\geq\frac{1}{4}\]
for all $g\in\MLP{1}{2}{1}$.
\end{thm}
\begin{proof}
For any function $g\in C(\R^2)$ we define
$$l_{g}(v):=\lim_{t\to\infty}\left(g(tv)+g(-tv)\right),$$
for the $v\in S^1\subseteq\R^2$ for which this limit exists. 

If $g(x)=c\varphi(a\cdot x+b)$, then $l_{g}$ will be defined on the full unit circle $S^1$, and constant almost everywhere, namely on all $v\in S^1$ satisfying $a\cdot v\neq0$ (or everywhere, if $a=0$). Taking a sum, linearity of limits implies that, for all $g\in\MLP{1}{2}{1}$, the partial function $l_g$ is defined everywhere and constant almost everywhere on $S^1$. (In fact, this holds for all $g\in\overline{\MLP{1}{2}{1}}$.) However,
\begin{align*}
l_f((\cos\theta,\sin\theta))=
\begin{cases}
1\qquad\theta\in(0,\tfrac12\pi)\cup(\pi,\tfrac32\pi)\\
0\qquad\theta\in(\tfrac12\pi,\pi)\cup(\tfrac32\pi,2\pi),
\end{cases}
\end{align*}
which is manifestly not constant almost everywhere. Hence, there exists a $v\in S^1$ with
$$\left|\lim_{t\to\infty}g(tv)-\lim_{t\to\infty}f(tv)\right|\geq\frac14,$$
which implies that $\|g-f\|_\infty\geq\frac{1}{4}$.
\end{proof}
The above reasoning for showing $f\notin\overline{\MLP{1}{n}{1}}$ will work for all $f\in C(\R^n)$ with $l_f$ not constant almost everywhere on $S^{n-1}$, hence supplying a large list of examples of functions not uniformly approximable by one-layer neural networks. By scaling, the uniform distance can be made arbitrary large, so there are functions in $\MLP{2}{n}{1}$ with arbitrarily large distance from $\MLP{1}{n}{1}$, proving Theorem \ref{thm:1 or 2 layers intro}.

\section{Nonzero one-layer networks do not vanish at infinity}
As an encore, we prove a claim made in the introduction.
We also prove a claim made in the caption of Figure \ref{fig:uniform approximation 1-layer increasing width}, which is stated without proof in \cite[Section 7]{Pinkus}. An elegant proof of the latter (i.e. the proof of point 2 below) was supplied by a very generous anonymous referee.
\begin{thm}\label{thm:0C0}
Let $m\in\N,n\in\N_{\geq2}$ be numbers, $a_1,\ldots,a_m\in\R^n$ be vectors, and $f_1,\ldots,f_m:\R\to\R$ be functions. Define $f(x):=\sum_{j=1}^mf_j(a_j\cdot x)$. 
\begin{enumerate}
\item\label{item:1} If $f\in C_0(\R^n)$ then $f=0$. In particular,
$$\MLP{1}{n}{1}\cap C_0(\R^n)=\{0\},$$ 
for every function $\varphi:\R\to\R$.
\item\label{item:2} If $f\in L^p(\R^n)$ for some $p\in(0,\infty)$, then $f=0$ almost everywhere. In particular,
$$\MLP{1}{n}{1}\cap L^p(\R^n)=\{0\},$$ 
for every function $\varphi:\R\to\R$ and every $p\in(0,\infty)$.
\end{enumerate}
\end{thm}
\begin{proof}
Since $n\geq2$, we can choose for each $j\in\{1,\ldots,m\}$ a vector $v_j\in\R^n$ satisfying
\[
a_j\cdot v_j=0\quad\text{and}\quad \|v_j\|>\|v_1\|+\ldots+\|v_{j-1}\|.
\]
For brevity, write $[m]:=\{1,\ldots,m\}$. For $I\subseteq[m]$, define
\[
v(I):=\sum_{i\in I}v_i\,,\quad\text{in particular,}\quad v(\emptyset)=0.
\]
By choice of the $v_j$ we have that $v(I)\neq0$ for all $\emptyset\neq I\subseteq[m]$.

For all $j\in[m]$, $x\in\R^n$, and $t\in\R$, we note that, since $a_j\cdot v_j=0$,
\[
\sum_{I\subseteq[m],j\notin I}(-1)^{\# I}f_j(a_j\cdot(x+tv(I)))=\sum_{I\subseteq[m],j\in I}(-1)^{\# I-1}f_j(a_j\cdot(x+t v(I))),
\]
and hence,
\begin{align*}
\sum_{I\subseteq[m]}(-1)^{\# I}f_j(a_j\cdot(x+tv(I)))=&\sum_{I\subseteq[m],j\in I}(-1)^{\# I}f_j(a_j\cdot(x+tv(I)))\\
&+\sum_{I\subseteq[m],j\notin I}(-1)^{\# I}f_j(a_j\cdot (x+tv(I)))\\
=&0.
\end{align*}
By summing the above over $j$ and interchanging the sums, we see $\sum_{I\subseteq[m]}(-1)^{\# I}f(x+tv(I))=0$. Since the summand for $I=\emptyset$ is simply $f(x)$, we thus get
\begin{align}\label{eq:combinatorial formula for 1-layer neural networks}
f(x)=-\sum_{\emptyset \neq I\subseteq[m]}(-1)^{\# I}f(x+tv(I))\qquad(x\in\R^n,t\in\R).
\end{align}
The above formula allows us to prove both \ref{item:1} and \ref{item:2}.
\begin{enumerate}
\item If $f\in C_0(\R^n)$ then for every $x\in\R^n$ and every $\epsilon>0$ there exists a $t\in\R$ large enough such that the right-hand side of \eqref{eq:combinatorial formula for 1-layer neural networks} is smaller than $\epsilon$, hence $|f(x)|<\epsilon$ for every $x$ and every $\epsilon$, implying $f=0$.
\item If $f\in L^p(\R^n)$ then $g=|f|^p\in L^1(\R^n)$. For completeness let us prove a folklore assertion related to the Poisson summation formula, namely the claim that for every $g\in L^1(\R^n)$ and any lattice $\Lambda\subseteq\R^n$, the series
\[
h(x):=\sum_{k\in\Lambda} g(x+k)
\]
converges for almost all $x\in\R^n$.
Indeed, if $U\subseteq \R^n$ is a measurable subset such that $\R^n=\sqcup_{k\in\Lambda}(U+k)$,
then
\[
\infty>\int_{\R^n}|g(x)|=\int_{U}\sum_{k\in\Lambda}|g(x+k)|\,dx\,.
\]
So there exists a null-set $N_0\subseteq U$ such that $h(x)<\infty$ for $x\in U\setminus N_0$. As $h(x+l)=h(x)$ for every $l\in\Lambda$, we have for the null-set $N_\Lambda:=N_0+\Lambda$ that $h(x)<\infty$ for every $x\in\R^n\setminus N_\Lambda$.

 For every non-empty $I\subseteq[m]$, $\Lambda:=\Z v(I)$ is a nontrivial lattice and so the above statement in particular supplies a null-set $N_I\subseteq\R^n$ such that
\[
\lim_{l\in\N,l\to\infty}g(x+lv(I))=0\qquad (x\in\R^n\setminus N_I).
\]
Hence,
\[
\lim_{l\in\N,l\to\infty}f(x+lv(I))=0\qquad (x\in\R^n\setminus N_I).
\]
Now, $N:=\bigcup_{\emptyset\neq I\subseteq[m]}N_I$ is a null-set, and \eqref{eq:combinatorial formula for 1-layer neural networks} shows for $x\in\R^n\setminus N$ that
\[
f(x)=-\sum_{\emptyset \neq I\subseteq[m]}(-1)^{\# I}f(x+lv(I))\to0\quad(l\in\N,l\to\infty).
\]
Thus, $f=0$ almost everywhere.
\end{enumerate}
\end{proof}

\section{Open questions}
As practitioners are asking for mathematically founded ways to choose the right architecture for their specific problems,
there is still much to learn regarding the influence of the number of neurons, the width, the depth, et cetera, on the approximation capabilities of neural networks. 
This paper indicates that the uniform topology on $\R^n$ is apt to gain new insights. Yet, there remains a lot of unexplored terrain.

Density theorems and analytic bounds expressing the quality of the optimal approximation in terms of the width and depth of the neural network, as well as the regularity of the activation function, have thus far been developed with respect to the topology of compact convergence and with respect to $L^p$-convergence \cite{Gripenberg,SYZ21,SYZ22,Yarotsky}. The spaces $C(\R^n)$ and $L^p(\R^n)$ naturally form the arenas of functions for which those bounds can be sought. By the present results, the spaces $C_0(\R^n)$, $\Cr(\R^n)$, and $\fS(\R^n)$ are the analogous arenas in which one should search for similar bounds and density theorems for global uniform convergence.

One concrete and fascinating question would be whether there exists an activation function such that every element of $C_0(\R^n)$ can be globally uniformly approximated up to arbitrary accuracy with a fixed width and depth dependent only on $n$, i.e., whether there exist globally superexpressive activations.

It would also be very informative to get theoretical bounds on the number of nodes that are needed to obtain a certain uniform precision (possibly constraining the amount of nodes per layer \cite{Gripenberg,KL} and also the amount of nodes in totality \cite[Section 5]{Ismailov}).

The same questions can be asked for specific classes of neural networks, such as convolutional neural networks or finite impulse recurrent neural networks, as then the `upper bound' of Proposition \ref{prop:first inclusion sigmoids} still holds. For instance, if a certain class of convolutional neural networks can be expressed as feedforward ANNs (the neural networks considered in this paper) with sigmoidal activation function, it follows that a function outside $\fS(\R^n)$ will not be uniformly approximated.
Whether the converse is true is not immediately clear.

A reason to favour the uniform norm over the $L^p$-norms has already been highlighted in the caption of Figure \ref{fig:uniform approximation 1-layer increasing width}. Namely, while the latter is infinite for all neural networks, the former can be finite and actually describe convergence to local ($C_0$) functions. One may also compare our results with those of weighted $L^p$ spaces, although they surrender translation invariance of the norm.
A related class of topologies is given by the weighted supremum norms, as considered in \cite{CST}. They form a natural class of topologies on $C(\R^n)$ that lie between the uniform topology and the topology of compact convergence, and it would therefore be interesting to relate the results of \cite{CST} to those obtained here.

On a different note, precisely because our results concern the uniform topology, they may yield quite tangible statements about the structure of a neural network \textit{after} training. For instance, the results of Sections \ref{sct:CR} and \ref{sct:sigmoids} imply that deep neural networks can be represented well as one-layer or two-layer neural networks without significant information loss at any scale, i.e., preserving the generalisation. Similarly, Theorem \ref{thm:main2} relates deep feedforward neural networks to Sum Of Product Neural Networks, as in \cite{LL,LWN}, which might be more light-weight than the corresponding deep neural network. It is unclear whether these observations lead to a practically feasible compression method, but the fact that there are ways to alternatively represent neural networks with arbitrary small loss deserves further investigation.

A final question is whether one can use the results of Theorem \ref{thm:1 or 2 layers intro} to reasonably decide whether to believe that an unknown model is a neural network, based on its responses to suitably chosen inputs.

The author expects further investigation of uniform universal approximation outside $[0,1]^n$ to be worthwhile, not in the least because it will require thinking outside the box.

\medskip

\noindent \textbf{Acknowledgements}\\
The author thanks Marjolein Troost for the crucial observations that sparked this paper, and for providing important guidance. The author is moreover indebted to Klaas Landsman, Walter van Suijlekom, Wim Wiegerinck, and the anonymous referees for helpful comments. 
This work was supported in part by NWO Physics Projectruimte 680-91-101, in part by NSF grant DMS-1554456, in part by ARC grant FL17010005, and in part by NWO grant OCENW.M.22.070.


\end{document}